\documentclass{article}

\usepackage{multirow}
\usepackage{microtype}
\usepackage{graphicx}
\usepackage{subfigure}
\usepackage{booktabs} %

\usepackage{hyperref}

\usepackage[accepted]{icml2024}

\usepackage{amsmath}
\usepackage{amssymb}
\usepackage{mathtools}
\usepackage{amsthm}

\usepackage[capitalize,noabbrev]{cleveref}

\usepackage[textsize=tiny]{todonotes}

\usepackage{thmtools}
\usepackage{thm-restate}

\newcommand{\stdv}[1]{\scriptsize$\pm$#1}
\newcommand{\cD}{\mathcal{D}}

\theoremstyle{plain}
\newtheorem{theorem}{Theorem}

\newtheorem{lemma}{Lemma}

\theoremstyle{definition}

\theoremstyle{remark}

\newtheorem{example}{Example}

\newcommand{\E}{\mathop{\mathbb{E}}}

\DeclareMathOperator*{\softmax}{softmax}

\usepackage{nccmath}
\renewcommand{\emptyset}{\phi}
\usepackage{algorithm}
\usepackage{algorithmic}
\renewcommand{\algorithmicrequire}{ \textbf{Input:}}

\icmltitlerunning{Provable Contrastive Continual Learning}

\begin{document}

\twocolumn[
\icmltitle{Provable Contrastive Continual Learning}

\icmlsetsymbol{equal}{*}

\begin{icmlauthorlist}
\icmlauthor{Yichen Wen}{equal,sjtu,sch1}
\icmlauthor{Zhiquan Tan}{equal,thu}
\icmlauthor{Kaipeng Zheng}{sjtu}
\icmlauthor{Chuanlong Xie}{sch1}
\icmlauthor{Weiran Huang \textsuperscript{\dag}}{sjtu,ailab}
\end{icmlauthorlist}

\icmlaffiliation{sjtu}{MIFA Lab, Qing Yuan Research Institute, SEIEE, Shanghai Jiao Tong University}
\icmlaffiliation{sch1}{Beijing Normal University}
\icmlaffiliation{thu}{Department of Mathematical Sciences, Tsinghua University}
\icmlaffiliation{ailab}{Shanghai AI Laboratory}

\icmlcorrespondingauthor{Weiran Huang}{weiran.huang@outlook.com}

\icmlkeywords{Machine Learning, ICML}

\vskip 0.3in
]

\printAffiliationsAndNotice{\icmlEqualContribution} %

\begin{abstract}
Continual learning requires learning incremental tasks with dynamic data distributions. So far, it has been observed that employing a combination of contrastive loss and distillation loss for training in continual learning yields strong performance. To the best of our knowledge, however, this contrastive continual learning framework lacks convincing theoretical explanations. In this work, we fill this gap by establishing theoretical performance guarantees, which reveal how the performance of the model is bounded by training losses of previous tasks in the contrastive continual learning framework. Our theoretical explanations further support the idea that pre-training can benefit continual learning. Inspired by our theoretical analysis of these guarantees, we propose a novel contrastive continual learning algorithm called CILA, which uses adaptive distillation coefficients for different tasks. These distillation coefficients are easily computed by the ratio between average distillation losses and average contrastive losses from previous tasks. Our method shows great improvement on standard benchmarks and achieves new state-of-the-art performance.
\end{abstract}

\section{Introduction}

Incrementally learning a sequence of tasks with dynamic data distributions is a typical setting for continual learning. We call the learned neural networks ``continual learners''. The main challenge for continual learners is to obtain a suitable trade-off between learning plasticity and memory stability. Specifically, excessive focus on learning plasticity of new tasks often leads to greatly reduced performance on old tasks \cite{McClelland1995WhyTA}, which is known as catastrophic forgetting. 

To address the challenge, the literature on continual learning has proposed various approaches. Representation-based approaches take advantage of representations. As one of these approaches, self-supervised learning with contrastive loss has demonstrated notable efficacy in obtaining robust representations against catastrophic forgetting in continual learning \cite{gallardo2021selfsupervised,fini2022selfsupervised}. For these methods based on contrastive loss, the training of representations is often decoupled with the training of the classifier, unlike methods based on cross-entropy. Specifically, contrastively trained representations suffer less catastrophic forgetting than ones trained by cross-entropy loss \cite{cha2021co2l}. Replay-based approaches use buffers to restore a part of previous data, and train networks using data from a combination of the current task and the buffer \cite{lopezpaz2017gradient}. Naturally, these methods are combined with knowledge distillation strategies to prevent the degradation of information in the network over time \cite{rebuffi2017icarl}. Regularization-based approaches introduce regularization terms to the target loss for continual learning to reach a balance between learning new tasks and preserving information from old tasks \cite{Kirkpatrick_2017}. Two main sub-directions within regularization-based approaches include weight regularization \cite{Ritter2018OnlineSL} and function regularization \cite{Li2016LearningWF}.

To achieve effective continual learning, a natural idea is to combine the three approaches above, and this idea leads to a new framework called contrastive continual learning \cite{cha2021co2l}, as illustrated in Figure~\ref{fig:1}. This framework focuses on using contrastively learned representations to learn new tasks and utilizing knowledge distillation to preserve information from past tasks, with the help of memory buffer and function regularization. The target loss of this framework contains a contrastive loss and a distillation loss with a distillation coefficient $\lambda$. The training data will be selected from the combination of the current data and buffered data. Empirically, this framework has been observed to be efficient, showing promising performance in continual learning \cite{cha2021co2l}. Despite the growing attention directed towards this framework, limited theoretical works have been proposed to explain its superior performance. 

In this paper, we try to address the theoretical problem of why this framework is efficient. Therefore, we consider the losses provided in \cite{cha2021co2l}. We have found a clear relationship between the contrastive losses of two consecutive models in continual learning. Inspired by this, we propose theoretical performance guarantees that reveal how the population test loss, i.e., the total performance of the final model on all seen tasks, is bounded by the series of training losses for the contrastive continual learning framework. Based on our theory, we propose a new and efficient contrastive continual learning algorithm called CILA, which uses distillation coefficients adapted to different tasks. Moreover, CILA consistently outperforms all baselines in different scenarios, datasets, and buffer sizes, e.g., about $1.77 \% $ improvement compared with the previous state-of-the-art method Co$^2$L \cite{cha2021co2l} on Seq-CIFAR-10 with a buffer of 500 samples for Class-IL scenario. 

Overall, our contributions are listed as follows. (1) We provide theoretical performance guarantees for the contrastive continual learning scheme. We identify that the overall performance of the final learned model on all seen tasks can be bounded by a function of the series of training losses with the distillation coefficient; (2) 
We propose an efficient algorithm CILA, which uses adaptive distillation coefficient $\lambda_t$ (replace $\lambda$ with $\lambda_t$ in Figure~\ref{fig:1}) for each task $t$; (3) 
We conduct extensive experiments to validate the efficacy of our algorithm, and the results strongly support our theory. Our method can inspire future works in contrastive continual learning.

\begin{figure*}
    \centering
    \includegraphics[width=\linewidth]{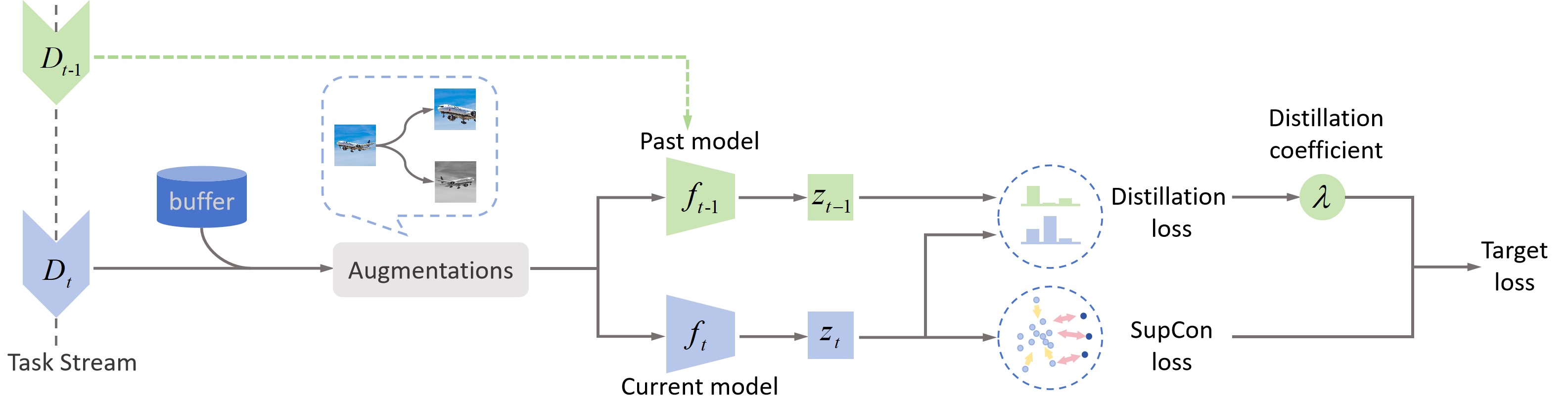}
    \caption{An illustration of contrastive continual learning framework. At the end of the previous task, we restore the previous model and values of losses. For the current task, augmentations are applied to both the buffered and the current data. Then the augmented data is passed through the current model and the previous frozen model to obtain representations. The target loss of contrastive continual learning is a weighted sum of contrastive loss and distillation loss with a distillation coefficient $\lambda$.}
    \label{fig:1}
\end{figure*}
\section{Related Work}

\textbf{Continual learning. }Continual learning is also referred to as incremental learning, which learns incremental tasks effectively \cite{wang2023comprehensive}.  The literature in this field mainly focuses on several streams including weight and function regularization \cite{Jung2020ContinualLW}, memory replay \cite{Prabhu2020GDumbAS}, sparse representations \cite{Javed2019MetaLearningRF}, parameter isolation \cite{Gurbuz2022NISPANS}, and dynamic architecture \cite{Ramesh2021ModelZA}.

As one of these effective continual learning methods, replay-based methods have demonstrated superior performance in terms of both learning plasticity and memory stability \cite{riemer2019learning}. Replay-based continual learning methods are developed from the idea of Experience Replay \cite{buzzega2020dark}, which typically stores past training samples in a fixed-size buffer. Currently, these replay-based methods are divided into two main streams, including experience replay and generative replay. Experience replay-based methods focus on the construction of memory buffer \cite{riemer2019learning,tiwari2022gcr} and storage efficiency \cite{caccia2019online,bang2021rainbow}. Generative replay-based methods concentrate on generative adversarial networks (GANs) to generate fine-grained data \cite{cong2020gan,ayub2021eec}. 

Representation-based methods for continual learning are also observed to be competitive. Recent works in continual learning take advantage of self-supervised learning to obtain robust representations, showing great performance on downstream tasks  \cite{pham2021dualnet}. Large-scale pre-training also contributes to improving transferable and robust representations for downstream continual learning \cite{gallardo2021selfsupervised,Ramasesh2022EffectOS}. 

Regularization-based methods mainly focus on weight and function regularization. Weight regularization methods add penalties to the loss function, typically the penalty is a quadratic one \cite{Kirkpatrick_2017,liu2018rotate}, and function regularization methods implement knowledge distillation on the intermediate or final output of the prediction function \cite{li2017learning,lee2019overcoming}. For function regularization methods, the teacher model is the frozen past model, and the student model is the current model.
Besides, there are some theory works analyzing regularization-based methods.
\citet{evron2022catastrophic} study the minimum norm estimator in CL under an over-parameterized and noise-free setup. \citet{li2023fixed} give a fixed design analysis of continual ridge regression for two-task linear regression.
\citet{zhao2024continual} consider a family of generalized $\ell_2$-regularization estimators and give some optimality analysis.

\textbf{Contrastive learning. }Contrastive learning aims to learn representations that attract different views of the same image while repelling views from different images \cite{tian2020contrastive}. Contrastive methods have been widely used in self-supervised learning and pre-training, showing superior performance on downstream tasks. The contrastive loss was first proposed in \cite{bromley1993signature} and then more formally defined in \cite{Chopra2005LearningAS} and \cite{Hadsell2006DimensionalityRB}. Later some theoretical analyses on the contrastive learning framework were provided in \cite{arora2019theoretical,huang2023towards, tan2023contrastive, tan2023otmatch, tan2023information, zhang2023kernel}. There are various target losses in contrastive learning, for example, InfoNCE loss \cite{oord2019representation} is a widely adopted and efficient one. Notably, methods in this field have reached or even outperformed supervised learning methods \cite{Khosla2020SupervisedCL} for image classification. Representative approaches include SimCLR \cite{chen2020simple}, MoCo v1\&v2 \cite{he2020momentum,chen2020improved}. In this work, we employ the contrastive loss provided in \cite{Khosla2020SupervisedCL} for the contrastive continual learning framework to show some theoretical insights.

\textbf{Knowledge distillation. }In various scenarios of continual learning, knowledge distillation is used to preserve information from the old model to the current model, contributing to mitigate catastrophic forgetting. Typically, knowledge distillation learns a small student model from a large teacher model with limited resources \cite{Gou_2021}. Various kinds of knowledge can be transferred by knowledge distillation, including response-based knowledge which is the neural response of the last output layer of the teacher model \cite{hinton2015distilling}, feature-based knowledge like feature maps \cite{zagoruyko2016paying} and relation-based knowledge referring to relationships between different layers or between different samples \cite{passalis2020heterogeneous}. Learning schemes of knowledge distillation include three streams, they are online distillation, offline distillation, and self-distillation. Among them, self-distillation considers the same structure between the teacher model and the student model \cite{zhang2020selfdistillation,mobahi2020selfdistillation}.

\section{Problem Setup}

We are given a sequence of $T$ supervised tasks,
with each task presented sequentially, one after the other.
For each task $t$, the training samples are assumed to be drawn from an unknown data distribution $\cD_t$. 
The model can be updated after seeing each task.
The goal of continual learning is to train a model $f$ that performs well over all seen tasks.

Supervised contrastive loss \cite{Khosla2020SupervisedCL} has shown its superiority over the cross-entropy loss in the standard supervised classification,
and it is then introduced into the continual learning by Co$^2$L \cite{cha2021co2l}.
Specifically, contrastive continual learning updates the model at each time step $t$ according to two losses, the contrastive loss and the distillation loss,
which measure the learning plasticity and memory stability, respectively.

\paragraph{Contrastive loss.}
For each task $t$, we use $\mu_t$ to denote the class distribution of task $\cD_t$,
and $\cD_c$ to denote the data distribution associated with each class $c$. In this paper, we consider a 
contrastive loss involving two similar samples $x,x^+$ i.i.d.\ drawn from the same class distribution $\cD_c$.
Meanwhile, there are several negative samples randomly picked from the whole data distribution $\cD_t$.
For simplicity, we only consider the case of one negative sample here, the case of multiple negative samples can be found in Appendix~\ref{appendixD} and~\ref{appendixE}.
Therefore, the contrastive loss can be formulated as
\begin{align*}
L_\text{con}(f;\cD_{t})= \E_{\substack{c^+ \sim \mu_t\\ c^- \sim \mu_t } }
\E_{\substack{x,x^+ \sim \cD_{c^+} \\
    x^- \sim \cD_{{c^-}}}} 
\ell \left[\medmath{f(x)^\top (f(x^+)\!-\!f(x^-))}\right]\!,
\end{align*}
where function $\ell(v)$ is defined as $\log (1+\exp(-v))$ and embeddings are conventionally normalized, i.e., $\|f\|=1$.

\paragraph{Distillation loss.}
Continual learning focuses on retaining previously acquired information while simultaneously learning new knowledge. In the specific context of contrastive continual learning, the model achieves knowledge preservation by keeping the model's ability to differentiate between similar and dissimilar (negative) samples.
To do so, we first compute the similarity probability distribution as
${\boldsymbol{p}} (f;x,x^+,x^-)=\softmax(f(x)^{\top} f(x^+),f(x)^{\top} f(x^-))$,
and then regulate the
cross-entropy between the past similarity probability distribution and the current one (e.g., IRD loss in \cite{cha2021co2l}). 
Specifically, for task $t$, we denote the distribution of all seen data by $\cD_{1:{t-1}}:= \sum_{j=1}^{t-1} k_{tj} \cD_j$, where we allow different tasks have different weights $k_{tj}>0$ with  $\sum_{j=1}^{t-1} k_{tj}=1$.
Therefore, the distillation loss considered in this paper can be formulated as 
\begin{align*}
&L_\text{dis}(f_t;f_{t-1},\cD_{1:t-1})= \\
&\E_{\substack{c^+\!\sim \mu_{1:t-1}\\ c^-\!\sim \mu_{1:t-1} } }
\E_{\substack{x,x^+\!\sim \cD_{c^+}\\
    x^-\!\sim \cD_{c^-}}} 
[- \medmath{{\boldsymbol{p}}(f_{t-1};x,x^+\!,x^-) \cdot\log{\boldsymbol{p}}(f_t;x,x^+\!,x^-)}],
\end{align*}
where $\mu_{1:t-1}$ represents the class distribution of $\cD_{1:t-1}$.

The total training loss of $f_t$ on task $t\ge 2$ is
\begin{align*}
&L_\text{train}(f_t;f_{t-1},\cD_t,\cD_{1:{t-1}})\\
& =L_\text{con}(f_t;\cD_t)+\lambda\cdot L_\text{dis}(f_t;f_{t-1},\cD_{1:{t-1}}),
\end{align*}
where $\lambda$ is a hyper-parameter for balancing the two loss terms.
For the first task $t=1$, the training loss does not have the distillation term, i.e., $L_\text{train}(f_1;\cD_1)=L_\text{con}(f_1;\cD_{1})$. 

To evaluate the contrastive continual learning model, we use the total performance (test loss) of the final model $f_T$ on all seen tasks, which can be formulated as
\begin{align*}
L_\text{test}(f_T;\cD_1,\dots,\cD_T) := \sum_{t=1}^T  L_\text{con}(f_T;\cD_{{t}}).
\end{align*}
\section{Theoretical Analysis}

Contrastive continual learning has demonstrated strong performance in practice. 
The focus of this paper is to examine its performance guarantees theoretically.
In particular, our study aims to investigate the relationship between the test loss $L_\text{test}(f_T;\cD_1,\dots,\cD_T)$ and the series of training losses $L_\text{train}(f_1;\cD_1)$, $L_\text{train}(f_2;f_1,\cD_2,\cD_{1})$, \dots, $L_\text{train}(f_T;f_{T-1},\cD_T,\cD_{1:T-1})$.

According to the definition, despite the distillation loss terms, the training losses involve $\{L_\text{con}(f_t;\cD_t)\}_{t=1}^T$, while the test loss consists of $\{L_\text{con}(f_T;\cD_t)\}_{t=1}^T$. 
To bridge the test loss and the training losses, we first
provide the relationship between the contrastive losses of two consecutive models $f_t$ and $f_{t-1}$ in the following lemma.

\begin{restatable}{lemma}{lma}
\label{lemma1}
When $t\ge 2$, for any data distribution $\cD$, the contrastive losses of
current model $f_t$ and previous model $f_{t-1}$ can be connected via the distillation loss, i.e.,
\begin{align*}
L_\text{\rm con}(f_t;\cD) &\le \alpha L_\text{\rm con}(f_{t-1};\cD) + L_\text{\rm dis}(f_t;f_{t-1},\cD) + \beta,\\
L_\text{\rm con}(f_t;\cD) &\ge \alpha L_\text{\rm con}(f_{t-1};\cD) + L_\text{\rm dis}(f_t;f_{t-1},\cD) + \beta',
\end{align*}
where
$\alpha = \frac{2e^2}{1+e^2}$, 
$\beta = 2-\alpha+\alpha \log \frac{\alpha}{2}$, and
$\beta' = -\alpha \log(1+e^2) - \alpha$.
\end{restatable}
The above lemma can be directly proved using the formulae for contrastive loss and distillation loss. A detailed proof is provided in the appendix due to the space limitation.

According to Lemma~\ref{lemma1}, when considering $\cD=\cD_{t}$ $(t \le T)$, a connection between $L_\text{\rm con}(f_T;\cD_t)$ and $L_\text{\rm con}(f_{T-1};\cD_t)$ can be established.
Similarly, a link between $L_\text{\rm con}(f_{T-1};\cD_t)$ and $L_\text{\rm con}(f_{T-2};\cD_t)$ can be drawn, and so on.
This approach allows us to build a bridge between $L_\text{\rm con}(f_T;\cD_t)$ and $L_\text{\rm con}(f_t;\cD_t)$ for any given $t$, which are the components of test loss and training losses, respectively.
Thus, with Lemma~\ref{lemma1}, we can now derive the relationship between the test loss $L_\text{test}(f_T;\cD_1,\dots,\cD_T)$ and the series of training losses $L_\text{train}(f_1;\cD_1)$, $L_\text{train}(f_2;f_1,\cD_2,\cD_{1})$, \dots, $L_\text{train}(f_T;f_{T-1},\cD_T,\cD_{1:T-1})$. Our results are presented in the following main theorem.

\begin{restatable}{theorem}{thmone}
\label{thm1}
For the contrastive continual learning involving $T\ge 2$ tasks, the test loss of the final model $f_T$ can be bounded via a linear combination of the training losses associated with each task. More specifically, the following two bounds are applicable.

(1) Upper bound:
\begin{align*}
&L_\text{\rm test}(f_T;\cD_1,\dots,\cD_T) \le \alpha^{T-1} L_\text{\rm train}(f_1;\cD_1) \\ &\quad + \sum_{t=2}^{T} \frac{\alpha^{T-t}}{\gamma_{t}(\lambda)} L_\text{\rm train}(f_t;f_{t-1},\cD_t,\cD_{1:{t-1}}) +\eta , 
\end{align*}

(2) Lower bound:
\begin{align*}
&L_\text{\rm test}(f_T;\cD_1,\dots,\cD_T) \ge \alpha^{T-1} L_\text{\rm train}(f_1;\cD_1)\\
&\quad + \sum_{t=2}^{T} \frac{\alpha^{T-t}}{\gamma_{t}'(\lambda)} L_\text{\rm train}(f_t;f_{t-1},\cD_t,\cD_{1:{t-1}}) +\eta' , 
\end{align*}
where
\begin{align*}
\begin{cases}
\alpha = \frac{2e^2}{ 1+e^2} ,\\
\gamma_{t}(\lambda) = \min\left(\{\frac{1}{t}\}\cup \{\lambda k_{tj}\}_{j=1}^{t-1}\right),   \\
\gamma_{t}'(\lambda) = \max \left(\{1\}\cup \{\lambda k_{tj}\}_{j=1}^{t-1}\right), \\
\eta =(2-\alpha+\alpha \log \frac{\alpha}{2}) \frac{T - 1 - T \alpha + (\alpha)^T}{(1-\alpha)^2}  \\
\qquad + \sum_{t=2}^{T} \alpha^{T-t} (1-\frac{1}{\gamma_{t}(\lambda)}) {\min}_f L_\text{\rm con}(f;\cD_t), \\
\eta' = -(\alpha \log(1+e^2) +\alpha) \frac{T - 1 - T \alpha + (\alpha)^T}{(1-\alpha)^2}.
\end{cases}
\end{align*}
\end{restatable}
The proof for the theorem can be found in the appendix.
It can be concluded from Theorem~\ref{thm1} that, the performance of the final model $f_T$ on all $T$ tasks, namely $L_\text{test}(f_T;\cD_1,\dots,\cD_T)$, can be well bounded by training losses on all seen tasks, suggesting that minimizing $L_\text{train}(f_t;f_{t-1},\cD_t,\cD_{1:{t-1}})$ during each task $t$ can help to improve the performance of the final model on all seen tasks. Note that there is also a lower bound of $L_\text{test}(f_T;\cD_1,\dots,\cD_T)$, which means that minimizing the training loss during each task $t$ is necessary. In particular, given that the training loss is a weighted sum of contrastive loss and distillation loss, these bounds also emphasize the necessity of both contrastive loss and distillation loss in effectively learning a contrastive continual learning model.

Taking inspiration from Theorem~\ref{thm1}, we can infer that pre-training can benefit continual learning. The coefficients of training losses associated with each task become fixed if $\lambda$ exceeds a certain value. For example, the denominators of the coefficients of training losses, i.e., $\{\gamma_{t}(\lambda)\}_{t=2}^T$ for the upper bound become constant values if $\lambda$ is large. Note that the component $ \alpha>1$, then the weight ${\alpha^{T-t}}/{\gamma_{t}(\lambda)}$ for the training loss of task $t$ decreases greatly as $t$ increases, reducing the importance of task $t$ in the bounds. Therefore, we have the following corollary, which shows that improved training performance of initial tasks contributes more to improving the later models' performance guarantees than that of later tasks. This aligns with the idea that pre-training can benefit continual learning, as observed in previous literature \cite{Wang2022MetaLearningWL,hu2022how}.

After choosing a suitable distillation coefficient $\lambda$, training performances of initial tasks in contrastive continual learning contribute more to improving the overall performance of the final model on all tasks compared with that of the latter ones, explaining that a well pre-trained network can benefit continual learning.

We conclude from the statement above that small changes in the training performance on the first task may lead to great changes in the overall performance of the final model. For example, the weight of $L_\text{train}(f_1;\cD_1)$ in the upper bound increases greatly when adding more tasks, and a large value of $L_\text{train}(f_1;\cD_1)$ implies a potential great increase of the upper bound. Therefore, well-trained initial models with small training losses in continual learning can be beneficial.

\section{Further Discussion on the Distillation Coefficient $\lambda$}
\subsection{Analysis on the distillation coefficient}

Inspired by additional analysis on Theorem~\ref{thm1}, we find that the suitable distillation coefficient $\lambda$ is correlated with 
the weights $\{\{k_{tj}\}_{j=1}^{t-1}\}_{t=1}^T$ that depends on the data distributions. Specifically, we would like to choose the suitable value of $\lambda$ as the turning point of the upper bound to get better theoretical guarantees. We define the turning point as the minimum value of $\lambda$ at which the upper bound no longer decreases. Once the distillation coefficient $\lambda$ exceeds the value of this turning point, the upper bound becomes a fixed value that is no longer influenced by $\lambda$, as illustrated in Figure~\ref{fig:2}. In the following part of this section, we will present several examples and calculate the corresponding turning point values. This may help us better understand the choice of distillation coefficient ($\lambda=1$) employed in the experiments of Co$^2$L \cite{cha2021co2l}.

We begin by clarifying that in the contrastive continual learning framework, choosing a fixed distillation coefficient as one for all tasks is favorable for achieving a balance between learning new tasks and preserving old knowledge. Specifically, with well-constructed weights $\{k_{tj}\}_{j=1}^{t-1}$ for task $t$, the suggested $\lambda$ value for learning tends to stay close to one, thereby contributing to a tighter upper bound. To illustrate this point, we provide an example below.
\begin{example}
\label{example1}
Assume that there are five tasks, each task with data distribution $\cD_t$, $t \in \{1,\dots,5\}$, and we have corresponding models $\{f_t\}_{t=1}^5$. We make the assumption that these models obtain the same value of training loss, i.e.,
\begin{align*}
L_\text{train}(f_5;f_{4},\mathcal{D}_5,\mathcal{D}_{1:4})= \cdots= L_\text{train}(f_1;\mathcal{D}_1).    
\end{align*}
Weights of tasks are given as follows, i.e., for $t \ge 2$,
\begin{align*}
k_{tj}=
\begin{cases}
\frac{2}{t}, & j=1,\\
\frac{1}{t}, & \text{else}.
\end{cases}
\end{align*}
These weights can be considered well-constructed, as they are uniform across different tasks, except for the weight of the first task which has a larger value than the others. Indeed, this strategy emphasizes the importance of the first task which is typically regarded as a base task.
\end{example}

Then according to our Theorem~\ref{thm1}, the value of appropriate $\lambda$ is suggested to be close to one to get a tighter upper bound on the overall performance for the final model. 

Note that in Example~\ref{example1}, we have assumed that the values of total training losses of different tasks remain the same, which may not align with realistic settings. To provide a more realistic illustration, we construct another example with adaptive ratios between the training losses of different tasks. Specifically, we allow the value of the training loss of each task to either increase or decrease with the same ratio $\rho$ close to one. The weights construction strategy is correlated to $\rho$ to maintain an alignment with the changing training loss. Constructed in this way, the following example illustrates how the appropriate value of $\lambda$ remains close to one even when $\rho$ fluctuates around one, and further achieves a tighter upper bound. This observation inspires us to consider adjusting the value of $\lambda$ around one.

\begin{figure}
    \centering
    \includegraphics[width=0.8\linewidth]{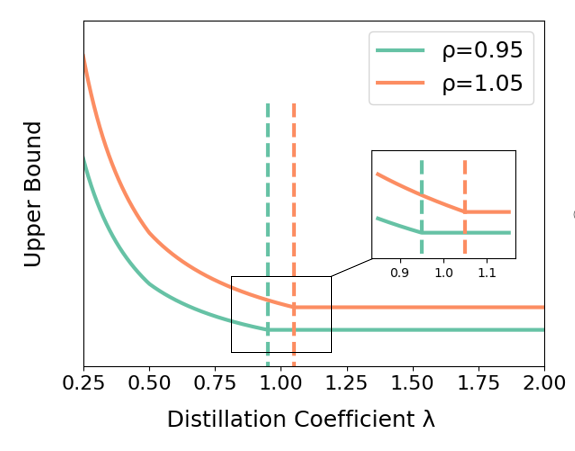}
    \caption{An illustration of Example~\ref{example2}. The suggesting $\lambda$ for $\rho=0.95$ or $\rho=1.05$ stays close to one. }
    \label{fig:2}
\end{figure}

\begin{example}
\label{example2}
Assume that there are five tasks, each task with data distribution $\cD_t$, $t \in \{1,\dots,5\}$, and we have corresponding models $\{f_t\}_{t=1}^5$. We assume that the value of the training loss of each task has a fixed ratio $\rho \approx 1$, i.e., 
\begin{align*}
L_\text{train}(f_5;f_{4},\mathcal{D}_5,\mathcal{D}_{1:4})= \cdots = \rho^4 L_\text{train}(f_1;\mathcal{D}_1).
\end{align*}
Weights of different tasks are given by a biased strategy related to $\rho$, i.e, for $t \ge 2$,
\begin{align*}
k_{tj}=
\begin{cases}
1-\frac{t-2}{\rho t}, & j=1, \\
\frac{1}{\rho t}, & \text{else}.
\end{cases}
\end{align*}
\end{example}

Then according to Theorem~\ref{thm1}, the value of appropriate $\lambda$ is suggested to get close to one as $\rho$ changes slightly around one, ensuring a tight upper bound on the overall performance of the final model. As illustrated in Figure~\ref{fig:2}, setting $\rho=0.95$ or $\rho=1.05$ implies that suggesting $\lambda$ value remains close to one. The settings in Example 2 are more realistic and can closely resemble the Co$^2$L configuration \cite{cha2021co2l}. Thus this example can further support the rationale for choosing $\lambda=1$ in Co$^2$L. 

However,  in an extreme case of the weights construction, an undesirable result may occur with a value of $\lambda$ greater than one. This insight suggests that data distribution may play a crucial role in selecting a suitable $\lambda$ since the weights $\{k_{tj}\}_{j=1}^{t-1}$ are strongly correlated with the data distribution of task $t$. As illustrated in the following example, when there exists a weight significantly smaller than the other weights, the value of appropriate $\lambda$ would deviate from one. In such a case, persistently using $\lambda=1$ may not achieve the best theoretical guarantees.
\begin{example}
\label{example3}
Take the same assumption from Example~\ref{example1}, excluding the weights construction strategy. Weights of tasks are given as follows, i.e., for $t \ge 3$,
\begin{align*}
k_{tj}=
\begin{cases}
\frac{2.9}{t}, & j=1, \\
\frac{0.1}{t}, & j=2, \\
\frac{1}{t}, & \text{else}.
\end{cases}
\end{align*}
For the second task, we set $k_{21}=1$. 
\end{example}

Then according to our Theorem~\ref{thm1}, the suggesting value of appropriate $\lambda$ is close to ten in Example~\ref{example3}. If we choose $\lambda=1$, the upper bound may get large and fail to provide the best guarantees. The failure of Example~\ref{example3} is attributed to a change in the weights construction strategy, highlighting a substantial relationship between the suitable value of distillation coefficient $\lambda$ and data distributions.

\subsection{Adaptive selection of distillation coefficients}

Inspired by our theoretical analysis, we are curious whether it is possible to provide better theoretical guarantees by dynamically adjusting distillation coefficients. 
Interestingly, the analysis of Theorem~\ref{thm1} can be adapted to the case of adaptive distillation coefficient $\lambda_t$, simply by replacing $\lambda$ with $\lambda_t$ for the coefficient of training loss of task $t$ in the bounds. Then we can conclude from Theorem~\ref{thm1} for the new case that, \emph{increasing} $\lambda_t$ for each task $t$ with a threshold strategy can provide better guarantees.

Note that our target is to get a set of distillation coefficients $\{\lambda_t\}_{t=2}^T$ that tighten the upper bound in Theorem~\ref{thm1}. Therefore, we want to adaptively select $\lambda_t$ for task $t$ to achieve this goal. We now propose our theoretical explanations for the adaptive selection of distillation coefficients. First, we give some related definitions. Suppose there are $T$ tasks for continual learning setting. For each task  $t \ge 2$, we denote $\lambda_t$ as the task-specific distillation coefficient, and define 
\begin{align*}
U_t= \sum_{j=2}^{t} \frac{\alpha^{t-j}}{\gamma_{j}(\lambda_t)} L_\text{train}(f_j;f_{j-1},\mathcal{D}_j,\mathcal{D}_{1:{j-1}}).
\end{align*}

Motivated by Theorem~\ref{thm1} and explanations above, at the end of task $t$, if the calculated $U_t$ has a relatively large value, then a slight increase in $\lambda_{t+1}$ around one can be beneficial for improving the upper bound. Inspired by this, we will maintain an extra set of task-specific threshold values $\{u_t\}_{t=2}^T$ where $u_t>0$, and a set of update momentums $\{\Delta_t\}_{t=1}^T$ where $ \Delta_t \ge 0$. After training task $t$, if $U_t > u_t$, we let $\lambda_{t+1}=\lambda_t+\Delta_t$, else, $\lambda_{t+1}=\lambda_t$. Then the total training loss of each task $t$ for model $f_t$ can be rewritten as 
\begin{align*}
&L_\text{train}(f_t;f_{t-1},\mathcal{D}_t,\mathcal{D}_{1:{t-1}}) \\
&= L_\text{con}(f_t;\mathcal{D}_t)+\lambda_t \cdot L_\text{dis}(f_t;f_{t-1},\mathcal{D}_{1:t-1}) .
\end{align*}

The following theorem provides a theoretical explanation for the benefits of the adaptive $\lambda_t$ selection protocol above.
\begin{restatable}{theorem}{thmtwo}
\label{thm2}
Assume that the training loss of each task is larger than zero. For each task $t\ge2$, there exists a task-specific constant $u_t>0$. If we have
\begin{align*}
U_t= \sum_{j=2}^{t} \frac{\alpha^{t-j}}{\gamma_{j}(\lambda_t)} L_\text{train}(f_j;f_{j-1},\mathcal{D}_j,\mathcal{D}_{1:{j-1}}) > u_t,
\end{align*}
where $\alpha$, $\{\gamma_{j}(\lambda_t)\}_{j=2}^t$ are defined in Theorem~\ref{thm1}, then we increase $\lambda_t$ by $\Delta_t$ , i.e., $\lambda_{t+1}=\lambda_t+\Delta_t$, else, $\lambda_{t+1}=\lambda_t$. By choosing $\lambda_t$ in this way, we get tighter upper bounds.
\end{restatable}

It can be concluded from Theorem~\ref{thm2} that, increasing $\lambda_t$ by a threshold strategy about the performance of the current model can help make the upper bound tighter. Moreover, Theorem~\ref{thm2} can also provide theoretical support for the choosing strategy of $\lambda$ in previous examples. 

\subsection{Contrastive Incremental Learning with Adaptive distillation (CILA)}

Note that we lack access to the construction of weights during the real training phase. Consequently, computing $U_t$ is not available for task $t$, and estimating $\{k_{tj}\}_{j=1}^{t-1}$ may be also inaccessible due to the high computing load and low estimating accuracy. However, according to our explanations above, it is suggested that the adaptive $\lambda_t$ for each task $t$ stays around one with a relatively larger value. Hence, we can find an easy-to-get and adaptive metric to replace $\lambda_t$. Before we introduce the chosen metric, we first give some related definitions. We first define the empirical contrastive loss \cite{Khosla2020SupervisedCL}. For each task $t$, we denote $D_t$ as the given batch of $N$ training samples $\{(x_{t,i},c_{t,i})\}_{i=1}^N$, and the augmented batch is $\{(\tilde{x}_{t,i},\tilde{c}_{t,i})\}_{i=1}^{2N}$ which is generated by making two randomly augmented versions of $x_{t,i}$ as $\tilde{x}_{t,2i-1}$ and $\tilde{x}_{t,2i}$ with $\tilde{c}_{t,2i-1}=\tilde{c}_{t,2i}=c_{t,i}$. The augmented samples are mapped to a $d$-dimensional Euclidean sphere by a model $f_t$. Denote $\boldsymbol{z}_{t,i}=f_t(\tilde{x}_{t,i})$, then the empirical contrastive loss can be formulated as
\begin{align*}
&\hat{L}_\text{con}(f_t;D_t) \\
&=\sum_{i=1}^{2N} \frac{-1}{|p_{t,i}|} \sum_{j \in p_{t,i}} \log \left( \frac{\exp(\boldsymbol{z}_{t,i}^\top \boldsymbol{z}_{t,j}/ \tau )}{\sum_{k \ne i} \exp(\boldsymbol{z}_{t,i}^\top \boldsymbol{z}_{t,k}/ \tau )}\right),
\end{align*}
where $p_{t,i}=\{j \in \{1,\dots,2N\}|j \ne i, c_{t,j}=c_{t,i}\}$ and $\tau>0$ is the temperature hyperparameter. Then we define the empirical distillation loss \cite{cha2021co2l}. We first define a similarity vector 
\begin{align*}
\boldsymbol{p}(f_t,\tau;\tilde{x}_i)=\text{softmax} &( \boldsymbol{z}_{t,i} ^\top \boldsymbol{z}_{t,1}/\tau, \dots, \boldsymbol{z}_{t,i} ^\top \boldsymbol{z}_{t,i-1}/\tau, \\
&\boldsymbol{z}_{t,i} ^\top \boldsymbol{z}_{t,i+1}/\tau, \dots \boldsymbol{z}_{t,i} ^\top \boldsymbol{z}_{t,2N}/\tau ).
\end{align*} 
Then the empirical distillation loss is formulated as
\begin{align*}
&\hat{L}_\text{dis}(f_t;f_{t-1},D_t) \\
&=\sum_{i=1}^{2N} -\boldsymbol{p}(f_{t-1},\tau^*;\tilde{x}_{t,i}) \cdot \log \boldsymbol{p}(f_t,\tau;\tilde{x}_{t,i}).
\end{align*}
Here, both $\tau^*$ and $\tau$ will remain fixed for all tasks. Actually, when constructing the experiment, we found that the ratio 
\begin{align*}
\sum^{t-1}_{j=2 }\hat{L}_\text{dis}(f_j;f_{j-1},D_j)/ \sum^{t-1}_{j=2 } \hat{L}_\text{con}(f_j;D_j), 
\end{align*}
is stable and stays close to one as task index $t\ge 3$  varies. This ratio is easy to get during the training procedure and is aligned with the idea that the distillation coefficient is suggested to stay close to one and varies according to the data distribution which depends on the task index. Therefore, inspired by Example~\ref{example2} and~\ref{example3}, an applicable method is to use  
\begin{align*}
\lambda_t
=\max (1, \kappa \sum^{t-1}_{j=2 }\hat{L}_\text{dis}(f_j;f_{j-1},D_j)/ \sum^{t-1}_{j=2 } \hat{L}_\text{con}(f_j;D_j) ),
\end{align*}
for task $t \ge 3$, where $\kappa$ is a balancing distillation coefficient, and for the second task, we use $\lambda_2=1$. More details can be found in Algorithm~\ref{algorithm}. In the following section, we will conduct several experiments.

\begin{algorithm}[t]
\caption{CILA: Contrastive Incremental Learning with Adaptive distillation}
\begin{algorithmic}[1]
\label{algorithm}

\REQUIRE Buffer size $B$, 
a sequence of training sets $\{D_t\}^T_{t=1}$,
base distillation coefficient $\lambda_0$,
balancing distillation coefficient $\kappa$.

\STATE  Initialize model $f_0$ and set buffer $\mathcal{M}\gets \emptyset$;
\FOR{task $t = 1, \cdots, T$}

\STATE Construct dataset ${D} \gets D_t \cup \mathcal{M}$;
\STATE Initialize model $f_t\gets f_{t-1}$;

\STATE Compute ${L}$ by ${L} \gets \hat{L}_\text{con}(f_{t};D)$; \\
\IF {$t > 1 $}
\STATE Adaptively update $\lambda_{t}$ by \\ $\lambda_t \gets \max(\lambda_0, \kappa \cdot \frac{\sum^{t-1}_{j=2 }\hat{L}_\text{dis}(f_j;f_{j-1},D_j)}{ \sum^{t-1}_{j=2 } \hat{L}_\text{con}(f_j;D_j) }) $; 
\STATE Update $ {L}$ by \\ $ {L} \gets  {L} + \lambda_{t} \cdot  \hat{L}_\text{dis}(f_{t};f_{t-1},D)$;
\ENDIF

\STATE Update 
$f_{t}$ by SGD;

\STATE Collect buffer samples until $|\mathcal{M}|=B$;
\ENDFOR
\end{algorithmic}
\end{algorithm}
\vspace{-0.5em}

\section{Experiment}

\begin{table*}[t]
\caption{Classification accuracies for Seq-CIFAR-10, Seq-Tiny-ImageNet, and R-MNIST on replay-based baselines and our algorithm. All results are averaged over ten independent trials. The best performance is marked as bold.
}
\setlength{\tabcolsep}{0.2cm}
\centering
\vskip 0.15in
\renewcommand{\arraystretch}{1.2}
\resizebox{\textwidth}{!}{%
\begin{tabular}{c|cccc|cccc|cc}
\toprule
Dataset & \multicolumn{4}{c|}{ Seq-CIFAR-10 } & \multicolumn{4}{c|}{ Seq-Tiny-ImageNet } &  \multicolumn{2}{c}{R-MNIST} \\
\midrule
Scenario & \multicolumn{2}{c}{Class-IL} & \multicolumn{2}{c|}{Task-IL}  & \multicolumn{2}{c}{Class-IL} & \multicolumn{2}{c|}{Task-IL} & \multicolumn{2}{c}{Domain-IL} \\
\midrule
Buffer & 200 & 500 & 200 & 500  & 200 & 500 & 200 & 500 & 200 & 500 \\
\midrule
ER 
& 44.79\stdv{1.86} & 57.74\stdv{0.27} 
& 91.19\stdv{0.94} & 93.61\stdv{0.27} 
& 8.49\stdv{0.16} & 9.99\stdv{0.29}  
& 38.17\stdv{2.00}  & 48.64\stdv{0.46}  
& 93.53\stdv{1.15}  & 94.89\stdv{0.95}  \\
GEM 
& 25.54\stdv{0.76} & 26.20\stdv{1.26} 
& 90.44\stdv{0.94} & 92.16\stdv{0.64} 
& -- & --
& --  & --  
& 89.86\stdv{1.23}  & 92.55\stdv{0.85}  \\
A-GEM 
& 20.04\stdv{0.34} & 22.67\stdv{0.57} 
& 83.88\stdv{1.49} & 89.48\stdv{1.45} 
& 8.07\stdv{0.08} & 8.06\stdv{0.04}  
& 22.77\stdv{0.03}  & 25.33\stdv{0.49}  
& 89.03\stdv{2.76}  & 89.04\stdv{7.01}  \\
iCaRL
& 49.02\stdv{3.20} & 47.55\stdv{3.95} 
& 88.99\stdv{2.13} & 88.22\stdv{2.62} 
& 7.53\stdv{0.79} & 9.38\stdv{1.53}  
& 28.19\stdv{1.47}  & 31.55\stdv{3.27}  
& --  & --  \\
FDR 
& 30.91\stdv{2.74} & 28.71\stdv{3.23} 
& 91.01\stdv{0.68} & 93.29\stdv{0.59} 
& 8.70\stdv{0.19} & 10.54\stdv{0.21}  
& 40.36\stdv{0.68}  & 49.88\stdv{0.71}  
& 93.71\stdv{1.51}  & 95.48\stdv{0.68}  \\
GSS
& 39.07\stdv{5.59} & 49.73\stdv{4.78} 
& 88.80\stdv{2.89} & 91.02\stdv{1.57} 
& -- & -- 
& --  & --  
& 87.10\stdv{7.23}  & 89.38\stdv{3.12}  \\
HAL 
& 32.36\stdv{2.70} & 41.79\stdv{4.46} 
& 82.51\stdv{3.20} & 84.54\stdv{2.36} 
& -- & -- 
& --  & --  
& 89.40\stdv{2.50}  & 92.35\stdv{0.81}  \\
DER 
& 61.93\stdv{1.79} & 70.51\stdv{1.67} 
& 91.40\stdv{0.92} & 93.40\stdv{0.39} 
& 11.87\stdv{0.78} & 17.75\stdv{1.14}  
& 40.22\stdv{0.67}  & 51.78\stdv{0.88}  
& 96.43\stdv{0.59}  & 97.57\stdv{1.47}  \\
DER++ 
& 64.88\stdv{1.17} & 72.70\stdv{1.36} 
& 91.92\stdv{0.60} & 93.88\stdv{0.50} 
& 10.96\stdv{1.17} & 19.38\stdv{1.41}  
& 40.87\stdv{1.16}  & 51.91\stdv{0.68}  
& 95.98\stdv{1.06}  & 97.54\stdv{0.43}  \\
Co$^{2}$L 
& 65.57\stdv{1.37} & 74.26\stdv{0.77} 
& 93.43\stdv{0.78} & 95.90\stdv{0.26} 
& 13.88\stdv{0.40} & 20.12\stdv{0.42}  
& 42.37\stdv{0.74}  & 53.04\stdv{0.69}  
& 97.90\stdv{1.92}  & 98.65\stdv{0.31}  \\
\textbf{CILA (Ours)}
& \textbf{67.06\stdv{1.59}} & \textbf{76.03\stdv{0.79}}  
& \textbf{94.29\stdv{0.24}}  & \textbf{96.40\stdv{0.21}} 
& \textbf{14.55\stdv{0.39}}  & \textbf{20.64\stdv{0.59}}  
& \textbf{44.15\stdv{0.70}}   & \textbf{54.13\stdv{0.72}}  
& \textbf{98.36\stdv{0.45}}   & \textbf{98.76\stdv{0.22}}   \\
\bottomrule
\end{tabular}}

\label{tab:main-table}
\vspace{-1em}
\end{table*}

\textbf{Learning settings and datasets. }We conducted experiments on three basic continual learning scenarios, Class-IL, Task-IL, and Domain-IL \cite{vandeVen2019ThreeSF}. Each scenario was evaluated using different datasets. Specifically, for Class-IL and Task-IL, we utilized Seq-CIFAR-10 and Seq-Tiny-ImageNet datasets. Seq-CIFAR-10 is a modified version of the CIFAR-10 \cite{Krizhevsky2009LearningML} dataset, where it is divided into 5 distinct subsets, each comprising two classes. Similarly, Seq-Tiny-ImageNet is an adapted version of the Tiny-ImageNet \cite{Le2015TinyIV} dataset, where the 200 classes are split into 10 separate sets, each containing 20 classes. The order of splits in Seq-CIFAR-10 and Seq-Tiny-ImageNet remains consistent across multiple runs.

For Domain-IL, we employed R-MINST, which is a variant of the MNIST \cite{Lecun1998Grad} dataset. In R-MINST, the original images are randomly rotated by an angle between 0 and $\pi$. R-MINST consists of 20 tasks, with each task corresponding to a randomly selected rotation angle. During the training process, samples from different tasks with the same digital class are treated as distinct classes.

In summary, our experiments covered Class-IL, Task-IL, and Domain-IL scenarios, utilizing Seq-CIFAR-10, Seq-Tiny-ImageNet, and R-MINST datasets, respectively.

\textbf{Baselines.}  
We compare our contrastive continual learning algorithm with replay-based continual learning baselines, including ER \cite{riemer2019learning}, GEM \cite{lopezpaz2017gradient}, A-GEM \cite{Chaudhry2018EfficientLL}, iCaRL \cite{rebuffi2017icarl}, FDR \cite{benjamin2019measuring}, GSS \cite{Aljundi2019GradientBS}, HAL \cite{Chaudhry2019UsingHT}, DER \cite{buzzega2020dark}, DER++ \cite{buzzega2020dark}, and Co$^2$L \cite{cha2021co2l}.

\textbf{Details of training. } Following the configuration of previous studies, we trained ResNet-18 on the Seq-CIFAR-10 and Tiny-ImageNet datasets. We implemented a simple network with convolution layers for the R-MNIST dataset. In our training process, we employed buffers of sizes 200 and 500. The base distillation coefficient $\lambda_0$ is set as one following the default configuration of Co$^2$L \cite{cha2021co2l}.

\textbf{Evaluation. }
Like Co$^2$L, CILA follows the idea of ``first pre-training, then linear probing''. Thus, unlike the joint representation-classifier training approaches, an additional classifier needs to be trained on top of the frozen representations. To ensure a fair comparison, the classifier is trained using only the samples from the last task and buffered samples, leveraging the representations learned by CILA. To mitigate the challenges posed by class imbalance, we employ a class-balanced sampling strategy during the training of a linear classifier. The strategy involves the following steps. We first uniformly select a class from the available set of classes. This ensures that each class has an equal chance of being chosen. Once a class is selected, we further uniformly sample an instance from that specific class. This guarantees that all instances within the chosen class are equal to be selected. 

For all experiments, a linear classifier is trained for a fixed number of epochs and we adopt 100 epochs to align with prior work. After training, the classification test accuracy is reported based on the predictions made by this classifier. 

\textbf{Main results.} In Table~\ref{tab:main-table}, our method outperforms all baselines in different scenarios, datasets, and buffer sizes, especially compared with Co$^2$L. This result verifies the superiority of our adaptive method and supports our theories strongly. Our algorithm successfully reaches a balance between learning plasticity and memory stability in continual learning. Under appropriate adaptation of the distillation coefficients, we also mitigate the catastrophic forgetting problem. Besides, the power of adaptation also impacts the performance of the learned continual learner, we will talk about it in the following ablation study.

\begin{table}
\caption{Accuracies on Seq-CIFAR-10 with 200 buffer samples.}
\setlength{\tabcolsep}{0.5cm}
\centering
\vskip 0.15in
\renewcommand{\arraystretch}{1.2}
\resizebox{0.8\linewidth}{!}{%
\begin{tabular}{cc|cc}
\toprule
Adaptive Method & Class-IL & Task-IL \\ \midrule
Co{$^2$}L & 65.57 & 93.43 \\
Pure-adapted       & 66.52 & 94.27 \\
Min-adapted      & 66.36 & 94.21 \\
Max-adapted & \textbf{67.06} & \textbf{94.29} \\ \bottomrule
\end{tabular} }
\label{tab:various adapted methods}
\end{table}

\section{Ablation Studies}

We conduct ablation experiments to verify the effectiveness of adaptive distillation coefficients. We consider two setups, Class-IL and Task-IL, and perform experiments on Seq-CIFAR-10 with three variants of adapted $\lambda_t$ for each task $t$. Variants include

(1) pure-adapted
\begin{align*}
\lambda_{t,\text{pure}}= \kappa \sum^{t-1}_{j=2 }\hat{L}_\text{dis}(f_j,f_{j-1};D_j)/ \sum^{t-1}_{j=2 } \hat{L}_\text{con}(f_j;D_j),
\end{align*}
(2) min-adapted $$\lambda_{t,\text{min}}=\min(1, \lambda_{t, pure}),$$ 
and (3) max-adapted $$\lambda_{t, \text{max}}=\max(1,  \lambda_{t,pure}),$$ 
where $\kappa$ is a balancing distillation coefficient. For our ablation experiments, we train the linear classifier on top of the representations with 200 buffer samples. 

As shown in Table~\ref{tab:various adapted methods}, methods with adaptive distillation coefficients show superior performance compared with the method with a fixed distillation coefficient with about $1\%$ improvement on both settings. As the adaptive $\lambda_t$ increases with moderate limitations, the performance of the model boosts with an obvious improvement on Class-IL. This verifies our assumption based on the theoretical results that continual learners with larger adaptive distillation coefficients show greater performance.

\section{Conclusion}

Contrastive learning has demonstrated remarkable performance in the field of continual learning, although there remains a lack of theoretical explanations. In this study, we aim to fill this gap by introducing theoretical performance guarantees for the final model in contrastive continual learning. Drawing inspirations from a detailed theoretical analysis, we propose the utilization of adaptive distillation coefficients for the distillation training loss in contrastive continual learning. Through comprehensive experiments conducted in diverse settings for continual learning, our approach surpasses baseline methods in terms of performance. We anticipate that our work can establish a robust foundation for continual learning from a representation perspective, and potentially spark further theoretical insights into the realm of contrastive continual learning.

\section*{Acknowledgment}
Weiran Huang is supported by 2023 CCF-Baidu Open Fund and Microsoft Research Asia.
Chuanlong Xie is supported by the National Nature Science Foundation of China (No.12201048).

We would also like to express our sincere gratitude to the reviewers of ICML 2024 for their insightful and constructive feedback. Their valuable comments have greatly contributed to improving the quality of our work.

\bibliography{reference}
\bibliographystyle{icml2024}

\clearpage
\appendix
\onecolumn
\begin{center}
    \Large \bf Appendix
\end{center}

\section{Proof of Lemma~\ref{lemma1}}
\label{prooflemma1}
We recall Lemma~\ref{lemma1}.
\lma*

\begin{proof}
For model $f$, denote the data pair $\boldsymbol{x}=(x,x^+,x^-)$ to simplify the proof, we first define 
$v(f;\boldsymbol{x})=f(x)^{\top} (f(x^+) - f(x^-))$, and
\begin{align*}
q(f;\boldsymbol{x}) =\frac{\exp(v(f;\boldsymbol{x}))}{1+\exp(v(f;\boldsymbol{x}))}.
\end{align*}
For models $f_t$ and $f_{t-1}$, the function $\ell(v)=\log(1+\exp(-v))$, and ${\boldsymbol{p}}(f;\boldsymbol{x})=\softmax(f(x)^\top f(x^+), f(x)^\top f(x^-))$, we have the following equation
\begin{align*}
& {-}{{\boldsymbol{p}}(f_{t-1};\boldsymbol{x})}\cdot{\log{\boldsymbol{p}}(f_t;\boldsymbol{x})} \\
& =-q(f_{t-1};\boldsymbol{x}) \log(q(f_{t};\boldsymbol{x}))-(1-q(f_{t-1};\boldsymbol{x})) \log(1-q(f_{t};\boldsymbol{x})) \\
& =q(f_{t-1};\boldsymbol{x}) \log(1+\exp(-v(f_t;\boldsymbol{x})))+(1-q(f_{t-1};\boldsymbol{x}))\log(1+\exp(v(f_t;\boldsymbol{x}))) \\
&=q(f_{t-1};\boldsymbol{x}) \log(1+\exp(-v(f_t;\boldsymbol{x})))+(1-q(f_{t-1};\boldsymbol{x}))\log(1+\exp(-v(f_t;\boldsymbol{x})))  +(1-q(f_{t-1};\boldsymbol{x}))\log(\exp(v(f_t;\boldsymbol{x}))) \\
&=\ell(v(f_t;\boldsymbol{x}))+(1-q(f_{t-1};\boldsymbol{x}))v(f_t;\boldsymbol{x}).
\end{align*}

Then for any data distribution $\mathcal{D}$, we have 
\begin{align*}
&L_\text{dis}(f_t;f_{t-1},\mathcal{D}) \\
&=\E_{\substack{c^+ \sim \mu\\ c^- \sim \mu } }
\E_{\substack{x,x^+ \sim \cD_{c^+} \\
    x^- \sim \cD_{{c^-}}}}  -{{\boldsymbol{p}}(f_{t-1};\boldsymbol{x})}\cdot{\log{\boldsymbol{p}}(f_t;\boldsymbol{x})} \\
&=\E_{\substack{c^+ \sim \mu\\ c^- \sim \mu } }
\E_{\substack{x,x^+ \sim \cD_{c^+} \\
    x^- \sim \cD_{{c^-}}}} \ell(v(f_t;\boldsymbol{x})) +\E_{\substack{c^+ \sim \mu\\ c^- \sim \mu } }
\E_{\substack{x,x^+ \sim \cD_{c^+} \\
    x^- \sim \cD_{{c^-}}}}  (1-q(f_{t-1};\boldsymbol{x}))v(f_t;\boldsymbol{x})
    \\
&= L_\text{con}(f_t;\mathcal{D})+\E_{\substack{c^+ \sim \mu\\ c^- \sim \mu } }
\E_{\substack{x,x^+ \sim \cD_{c^+} \\
    x^- \sim \cD_{{c^-}}}}  (1-q(f_{t-1};\boldsymbol{x}))v(f_t;\boldsymbol{x}). 
\end{align*}

Note that $f_t,f_{t-1}$ are normalized, i.e., 
$-2 \le v(f_{t-1};\boldsymbol{x}) \le 2$, and
$-2 \le v(f_t;\boldsymbol{x}) \le 2$. We first prove the upper bound. By using the following inequality
\begin{align*}
\alpha \log(1+e^{-h}) + \beta - \frac{2}{1+e^h} \ge 0,
\end{align*}
where $\alpha = \frac{2e^2}{1+e^2}$, $\beta = 2-\alpha +\alpha  \log \frac{\alpha}{2} $, $-2 \le h \le 2$, and using $v(f_{t-1};\boldsymbol{x})$ to replace $h$, we have  
\begin{align*}
(1-q(f_{t-1};\boldsymbol{x})) v(f_t;\boldsymbol{x}) &  \ge -2 (1-q(f_{t-1};\boldsymbol{x})) \\
&  = - \frac{2}{1+\exp(v(f_{t-1};\boldsymbol{x}))} \\
&  \ge - \alpha \log(1+\exp(-v(f_{t-1};\boldsymbol{x}))) - \beta \\
&  = - \alpha \ell(v(f_{t-1};\boldsymbol{x})) - \beta.
\end{align*}

Using the result above, we have 
\begin{align*}
L_\text{dis}(f_t;f_{t-1},\mathcal{D}) 
&=L_\text{con}(f_t;\mathcal{D})+\E_{\substack{c^+ \sim \mu\\ c^- \sim \mu } }
\E_{\substack{x,x^+ \sim \cD_{c^+} \\
    x^- \sim \cD_{{c^-}}}}  (1-q(f_{t-1};\boldsymbol{x}))v(f_t;\boldsymbol{x})\\
& \ge L_\text{con}(f_t;\mathcal{D})+\E_{\substack{c^+ \sim \mu\\ c^- \sim \mu } } \E_{\substack{x,x^+ \sim \cD_{c^+} \\     x^- \sim \cD_{{c^-}}}} [ -  \alpha \ell(v(f_{t-1};\boldsymbol{x})) - \beta ]\\
& = L_\text{con}(f_t;\mathcal{D}) - \alpha L_\text{con}(f_{t-1};\mathcal{D}) - \beta.
\end{align*}
which means
\begin{align*}
L_\text{con}(f_t;\mathcal{D}) \le \alpha L_\text{con}(f_{t-1};\mathcal{D}) + L_\text{dis}(f_t;f_{t-1},\mathcal{D}) + \beta.
\end{align*}
This finishes the upper bound part. Then let us consider the lower bound.  We use the inequality
\begin{align*}
\alpha \log(1+e^{-h}) + \beta' + \frac{2}{1+e^h} \le 0,
\end{align*}
where $\alpha = \frac{2e^2}{1+e^2}$, $\beta' = -\alpha \log(1+e^2) - \alpha$, $-2 \le h \le 2$, and using $v(f_{t-1};\boldsymbol{x})$ to replace $h$, then we have  
\begin{align*}
(1-q(f_{t-1};\boldsymbol{x})) v(f_t;\boldsymbol{x})
& \le 2 (1-q(f_{t-1};\boldsymbol{x})) \\
& =  \frac{2}{1+\exp(v(f_{t-1};\boldsymbol{x}))} \\
& \le - \alpha \log(1+\exp(-v(f_{t-1};\boldsymbol{x})) - \beta' \\
& = - \alpha \ell(v(f_{t-1};\boldsymbol{x})) - \beta'.
\end{align*}

Using the result above, we have
\begin{align*}
L_\text{dis}(f_t;f_{t-1},\mathcal{D}) 
&=L_\text{con}(f_t;\mathcal{D})+\E_{\substack{c^+ \sim \mu\\ c^- \sim \mu } }
\E_{\substack{x,x^+ \sim \cD_{c^+} \\
    x^- \sim \cD_{{c^-}}}}  (1-q(f_{t-1};\boldsymbol{x}))v(f_t;\boldsymbol{x})\\
& \le L_\text{con}(f_t;\mathcal{D})+\E_{\substack{c^+ \sim \mu\\ c^- \sim \mu } } \E_{\substack{x,x^+ \sim \cD_{c^+} \\     x^- \sim \cD_{{c^-}}}} [ -  \alpha \ell(v(f_{t-1};\boldsymbol{x})) - \beta ']\\
& = L_\text{con}(f_t;\mathcal{D}) - \alpha L_\text{con}(f_{t-1};\mathcal{D}) - \beta'.\\
\end{align*} 
It can be translated into
\begin{align*}
L_\text{con}(f_t;\mathcal{D}) \ge \alpha L_\text{con}(f_{t-1};\mathcal{D}) + L_\text{dis}(f_t;f_{t-1},\mathcal{D}) + \beta'.
\end{align*}
The proof has finished.
\end{proof}

\section{Proof of Theorem~\ref{thm1}}
\label{proofthm1}
We recall Theorem~\ref{thm1}.
\thmone*

\begin{proof}
We first proof the upper bound. For models $f_t$ and $f_{t-1}$, we have
\begin{align*}
 L_\text{dis}(f_t;f_{t-1},\mathcal{D}_{1:t-1})
& = \E_{\substack{c^+ \sim \mu_{1:{t-1}}\\ c^- \sim \mu_{1:{t-1}} } } \E_{\substack{x,x^+ \sim \cD_{c^+} \\ 
x^- \sim \cD_{{c^-}}}}  
[{-}{{\boldsymbol{p}}(f_{t-1};\boldsymbol{x}})\cdot{\log{\boldsymbol{p}}(f_t;\boldsymbol{x})}] \\
& = \sum_{j=1}^{t-1} k_{tj} 
\E_{\substack{c^+ \sim \mu_j\\ c^- \sim \mu_j } } \E_{\substack{x,x^+ \sim \cD_{c^+} \\ 
x^- \sim \cD_{{c^-}}}}
[{-}{{\boldsymbol{p}}(f_{t-1};\boldsymbol{x}})\cdot{\log{\boldsymbol{p}}(f_t;\boldsymbol{x})}] \\
& = \sum_{j=1}^{t-1} k_{tj} L_\text{dis}(f_t;f_{t-1},\mathcal{D}_j).
\end{align*}

Then we can write the training loss as
\begin{align*}
L_\text{train}(f_t;f_{t-1},\mathcal{D}_t, \mathcal{D}_{1:t-1}) 
&=L_\text{con}(f_t;\mathcal{D}_t)+ \lambda L_\text{dis}(f_t;f_{t-1},\mathcal{D}_{1:t-1}) \\
&= L_\text{con}(f_t;\mathcal{D}_t) + \lambda \sum_{j=1}^{t-1} k_{tj} L_\text{dis}(f_t;f_{t-1},\mathcal{D}_j).
\end{align*}
for task $t \ge 2$, and $L_\text{train}(f_1;\mathcal{D}_1)=L_\text{con}(f_1;\mathcal{D}_1)$.
According to the proof of Lemma~\ref{lemma1}, for models $f_t$ and $f_{t-1}$, data distribution $\mathcal{D}_j$ $ (j \le t)$, we have
\begin{align*}
L_\text{con}(f_t,\mathcal{D}_j) \le \alpha L_\text{con}(f_{t-1},\mathcal{D}_j) + L_\text{dis}(f_t;f_{t-1},\mathcal{D}_j) + \beta.
\end{align*}
Denote  $\gamma_{t}(\lambda) = \min\left(\{\frac{1}{t}\}\cup \{\lambda k_{tj}\}_{j=1}^{t-1}\right)$ for task $t \ge 2$. Then we have
\begin{align*}
& L_\text{test}(f_T;\mathcal{D}_{1:T}) \\
& = L_\text{con}(f_T;\mathcal{D}_T) + \sum_{t=1}^{T-1}  L_\text{con}(f_T,\mathcal{D}_t) \\
& \le L_\text{con}(f_T;\mathcal{D}_T) + \sum_{t=1}^{T-1} [L_\text{dis}(f_T;f_{T-1},\mathcal{D}_t) + \alpha L_\text{con}(f_{T-1};\mathcal{D}_t) + \beta] \\
& \le (\frac{1}{\gamma_T(\lambda)} - \frac{1}{\gamma_T(\lambda)} +1)L_\text{con}(f_T;\mathcal{D}_T) + \frac{\lambda}{\gamma_T(\lambda)} \sum_{t=1}^{T-1} k_{Tt} L_\text{dis}(f_T;f_{T-1},\mathcal{D}_t) + \sum_{t=1}^{T-1} [\alpha L_\text{con}(f_{T-1};\mathcal{D}_t) + \beta] \\
& \le \frac{1}{\gamma_T(\lambda)} L_\text{train}(f_T;f_{T-1},\mathcal{D}_T,\mathcal{D}_{1:T-1})+(1- \frac{1}{\gamma_T(\lambda)}) {\min}_f L_\text{con}(f;\mathcal{D}_T) \\
& \qquad + (T-1)\beta + \alpha L_\text{test}(f_{T-1};\mathcal{D}_{1:{T-1}})\\
&\qquad \vdots \\
& \le  \alpha^{T-1} L_\text{train}(f_1;\mathcal{D}_1) +\sum_{t=2}^{T} \frac{\alpha^{T-t}}{\gamma_{t}(\lambda)} L_\text{train}(f_t;f_{t-1},\mathcal{D}_t,\mathcal{D}_{1:t-1})+\eta. 
\end{align*}
where
$\alpha = \frac{2e^2}{1+e^2}$, $\eta =(2-\alpha+\alpha \log \frac{\alpha}{2}) \frac{T - 1 - T \alpha + (\alpha)^T}{(1-\alpha)^2}+\sum_{t=2}^{T} \alpha^{T-t} (1-\frac{1}{\gamma_{t}(\lambda)}) {\min}_f L_\text{con}(f;\mathcal{D}_t)$.

Let us prove the lower bound. According to the proof of Lemma~\ref{lemma1}, for models $f_t$ and $f_{t-1}$, and data distribution $\mathcal{D}_j$ $ (j \le t)$, we have
\begin{align*}
L_\text{con}(f_t,\mathcal{D}_j) \ge \alpha L_\text{con}(f_{t-1},\mathcal{D}_j) + L_\text{dis}(f_t;f_{t-1},\mathcal{D}_j) + \beta' .
\end{align*}
Denote $\gamma_{t}'(\lambda) = \max\left(\{1\}\cup \{\lambda k_{tj}\}_{j=1}^{t-1}\right)$ for task $t \ge 2$. The proof is similar to that of the upper bound.
\begin{align*}
& L_\text{test}(f_T;\mathcal{D}_{1:T}) \\
& = L_\text{con}(f_T;\mathcal{D}_T) + \sum_{t=1}^{T-1}  L_\text{con}(f_T,\mathcal{D}_t) \\
& \ge L_\text{con}(f_T;\mathcal{D}_T) + \sum_{t=1}^{T-1} [L_\text{dis}(f_T;f_{T-1},\mathcal{D}_t) + \alpha L_\text{con}(f_{T-1};\mathcal{D}_t) + \beta'] \\
& \ge \frac{1}{\gamma_t'(\lambda)}[ L_\text{con}(f_T;\mathcal{D}_T) + \lambda \sum_{t=1}^{T-1} k_{Tt} L_\text{dis}(f_T;f_{T-1},\mathcal{D}_t)] + \sum_{t=1}^{T-1} [\alpha L_\text{con}(f_{T-1};\mathcal{D}_t) + \beta'] \\
& = \frac{1}{\gamma_t'(\lambda)} L_\text{train}(f_T;f_{T-1},\mathcal{D}_T,\mathcal{D}_{1:T-1}) + \alpha \sum_{t=1}^{T-1} L_\text{con}(f_{T-1};\mathcal{D}_t) + (T-1)\beta'\\
& =\frac{1}{\gamma_t'(\lambda)} L_\text{train}(f_T;f_{T-1},\mathcal{D}_T,\mathcal{D}_{1:T-1}) + \alpha L_\text{test}(f_{T-1};\mathcal{D}_{1:{T-1}}) + (T-1)\beta' \\
& \ge  \alpha^{T-1} L_\text{train}(f_1;\mathcal{D}_1)+\sum_{t=2}^{T} \frac{\alpha^{T-t}}{\gamma_t'(\lambda)} L_\text{train}(f_t;f_{t-1},\mathcal{D}_t,\mathcal{D}_{1:t-1}) +\eta'.
\end{align*}
where
$\alpha = \frac{2e^2}{1+e^2}$, $\eta' = - (\alpha \log(1+e^2) + \alpha) \frac{T - 1 - T \alpha + (\alpha)^T}{(1-\alpha)^2}$.
\end{proof}

\section{Proof of Theorem~\ref{thm2}} 
\label{proofthm2}
We recall Theorem~\ref{thm2}.
\thmtwo*

\begin{proof}
Note that $U_t>0$, thus $u_t$ exists. If $\lambda_t$ increases by $\Delta \ge 0$, then $\lambda_{t+1} \ge \lambda_t$ and $\gamma_{j}(\lambda_{t+1}) \ge \gamma_{j}(\lambda_{t})$. We have
\begin{align*}
U_{t+1}'= \sum_{j=2}^{t+1} \frac{\alpha^{t+1-j}}{\gamma_{j}(\lambda_{t+1})} L_\text{train}(f_j;f_{j-1},\mathcal{D}_j) \le U_{t+1}=\sum_{j=2}^{t+1} \frac{\alpha^{t+1-j}}{\gamma_{j}(\lambda_{t})} L_\text{train}(f_j;f_{j-1},\mathcal{D}_j) .
\end{align*}
Thus the upper bound becomes tighter.
\end{proof}

\section{The Case of Multiple Negative Examples for Lemma~\ref{lemma1}}
\label{appendixD}

Following our definitions in the paper, the contrastive loss for the case of $k(k\ge 1)$ negative samples can be formulated as
\begin{align*}
L_\text{con}(f;\cD_{t})= \E_{\substack{c^+ \sim \mu_t\\ c_i^- \sim \mu_t } }
\E_{\substack{x,x^+ \sim \cD_{c^+} \\
    x_i^- \sim \cD_{{c_i^-}}}} 
\ell \left[\{\medmath{f(x)^\top (f(x^+)\!-\!f(x_i^-))}\}\right]\!,
\end{align*}
where function $\ell(\boldsymbol{v})$ is defined as $\ell(\boldsymbol{v})=\log(1+\sum_{i=1}^k \exp(-{v_i}))$ for $\boldsymbol{v} \in \mathbb{R}^k$ and embeddings are conventionally normalized, i.e., $\|f\|=1$.
The distillation loss for the case of $k$ negative samples can be formulated as 
\begin{align*}
&L_\text{dis}(f_t;f_{t-1},\cD_{1:t-1})= \\
&\E_{\substack{c^+\!\sim \mu_{1:t-1}\\ c_i^-\!\sim \mu_{1:t-1} } }
\E_{\substack{x,x^+\!\sim \cD_{c^+}\\
    x_i^-\!\sim \cD_{c_i^-}}} 
[- \medmath{{\boldsymbol{p}}(f_{t-1};x,x^+\!,x_1^-,\dots,x_k^-) \cdot\log{\boldsymbol{p}}(f_{t};x,x^+\!,x_1^-,\dots,x_k^-)}],
\end{align*}
where $\mu_{1:t-1}$ represents the class distribution of $\cD_{1:t-1}$ and 
\begin{align*}
{\boldsymbol{p}}(f;x,x^+\!,x_1^-,\dots,x_k^-)=\softmax(f(x)^\top f(x^+), f(x)^\top f(x_1^-), \dots, f(x)^\top f(x_k^-)).  
\end{align*}

Then we provide the extended version of Lemma~\ref{lemma1} and its proof.
\begin{lemma}
\label{lemma1ex}
When $t\ge 2$ and the number of negative samples $k\ge 1$, for any data distribution $\cD$, the contrastive losses of
current model $f_t$ and previous model $f_{t-1}$ can be connected via the distillation loss, i.e.,
\begin{align*}
{L_\text{\rm con}(f_t;\cD) \le \alpha L_\text{\rm con}(f_{t-1};\cD) + L_\text{\rm dis}(f_t;f_{t-1},\cD) + \beta,}\\
{L_\text{\rm con}(f_t;\cD) \ge \alpha L_\text{\rm con}(f_{t-1};\cD) + L_\text{\rm dis}(f_t;f_{t-1},\cD) + \beta',}
\end{align*}
where
$\alpha = \frac{2e^2}{k+e^2}$, 
$\beta = 2-\alpha+\alpha \log \frac{\alpha}{2}$, and
$\beta'  = -\alpha \log(1+ke^2) -\frac{2ke^2}{1+ke^2}$.
\end{lemma}

\begin{proof}
For model $f$, denote the data pair $\boldsymbol{x}=(x,x^+,x_1^-,\dots,x_k^-)$ to simplify the proof, we first define 
${v}_i(f;\boldsymbol{x})=f(x)^{\top} (f(x^+) - f(x_i^-))$, and
\begin{align*}
q_i(f;\boldsymbol{x}) =\frac{\exp(-v_i(f;\boldsymbol{x}))}{1+\sum_{i=1}^k \exp(-v_i(f;\boldsymbol{x}))}, \\
q(f;\boldsymbol{x}) =\frac{1}{1+\sum_{i=1}^k \exp(-v_i(f;\boldsymbol{x}))},
\end{align*}
where $q(f;\boldsymbol{x})+\sum_{i=1}^k q_i(f;\boldsymbol{x}) =1$.
For models $f_t$ and $f_{t-1}$, the function $\ell(\boldsymbol{v})=\log(1+\sum_{i=1}^k\exp(-{v_i}))$ for $\boldsymbol{v} \in \mathbb{R}^k$ , and 
\begin{align*}
{\boldsymbol{p}}(f;\boldsymbol{x})=\softmax(f(x)^\top f(x^+), f(x)^\top f(x_1^-), \dots, f(x)^\top f(x_k^-)),    
\end{align*}
we have the following equation
\begin{align*}
& {-}{{\boldsymbol{p}}(f_{t-1};\boldsymbol{x})}\cdot{\log{\boldsymbol{p}}(f_t;\boldsymbol{x})} \notag\\
& =-q(f_{t-1};\boldsymbol{x}) \log(q(f_{t};\boldsymbol{x}))-\sum_{i=1}^k q_i(f_{t-1};\boldsymbol{x}) \log(q_i(f_{t};\boldsymbol{x}))\notag \\
& =q(f_{t-1};\boldsymbol{x}) \log(1+\sum_{i=1}^k \exp(-v_i(f_t;\boldsymbol{x})))- \sum_{i=1}^k q_i(f_{t-1};\boldsymbol{x}) [-v_i(f_t;\boldsymbol{x})-\log(1+\sum_{i=1}^k \exp(-v_i(f_t;\boldsymbol{x}))) ]\notag\\
& =q(f_{t-1};\boldsymbol{x}) \log(1+\sum_{i=1}^k \exp(-v_i(f_t;\boldsymbol{x})))+ \sum_{i=1}^k q_i(f_{t-1};\boldsymbol{x}) [v_i(f_t;\boldsymbol{x})+\log(1+\sum_{i=1}^k \exp(-v_i(f_t;\boldsymbol{x}))) ]\notag\\
& =\log(1+\sum_{i=1}^k \exp(-v_i(f_t;\boldsymbol{x})))+ \sum_{i=1}^k q_i(f_{t-1};\boldsymbol{x}) {v}_i(f_t;\boldsymbol{x}) \notag\\
&=\ell(\boldsymbol{v}(f_t;\boldsymbol{x}))+ \sum_{i=1}^k q_i(f_{t-1};\boldsymbol{x}) {v}_i(f_t;\boldsymbol{x}).
\end{align*}
Then for any data distribution $\mathcal{D}$, we have 
\begin{align*}
&L_\text{dis}(f_t;f_{t-1},\mathcal{D}) \notag\\
&=\E_{\substack{c^+ \sim \mu\\ c_i^- \sim \mu } }
\E_{\substack{x,x^+ \sim \cD_{c^+} \\
    x_i^- \sim \cD_{{c_i^-}}}}  -{{\boldsymbol{p}}(f_{t-1};\boldsymbol{x})}\cdot{\log{\boldsymbol{p}}(f_t;\boldsymbol{x})} \notag\\
&=\E_{\substack{c^+ \sim \mu\\ c_i^- \sim \mu } }
\E_{\substack{x,x^+ \sim \cD_{c^+} \\
    x_i^- \sim \cD_{{c_i^-}}}} \ell(\boldsymbol{v}(f_t;\boldsymbol{x})) +\E_{\substack{c^+ \sim \mu\\ c_i^- \sim \mu } }
\E_{\substack{x,x^+ \sim \cD_{c^+} \\
    x_i^- \sim \cD_{{c_i^-}}}}  \sum_{i=1}^k q_i(f_{t-1};\boldsymbol{x}) v_i(f_t;\boldsymbol{x})
    \notag\\
&= L_\text{con}(f_t;\mathcal{D})+\E_{\substack{c^+ \sim \mu\\ c_i^- \sim \mu } }
\E_{\substack{x,x^+ \sim \cD_{c^+} \\
    x_i^- \sim \cD_{{c_i^-}}}}  \sum_{i=1}^k q_i(f_{t-1};\boldsymbol{x}) v_i(f_t;\boldsymbol{x}). 
\end{align*}
Note that $f_t,f_{t-1}$ are normalized, i.e., 
$-2 \le v_i(f_{t-1};\boldsymbol{x}) \le 2$, and
$-2 \le v_i(f_t;\boldsymbol{x}) \le 2$. We first prove the upper bound. By using the following inequality
\begin{align*}
\alpha \log h + \beta  \ge 2(1-\frac{1}{h}),
\end{align*}
where $\alpha = \frac{2e^2}{k+e^2}$, $\beta = 2-\alpha +\alpha  \log \frac{\alpha}{2}$, $1+ke^{-2} \le h \le 1+ke^{2}$, and using $1+\sum_{i=1}^k \exp(-v_i(f_{t-1};\boldsymbol{x}))$ to replace $h$ in the inequality above, we have  
\begin{align*}
\sum_{i=1}^k q_i(f_{t-1};\boldsymbol{x}) v_i(f_t;\boldsymbol{x}) &  \ge -2 (1-q(f_{t-1};\boldsymbol{x})) \notag\\
&  = -2(1- \frac{1}{1+\sum_{i=1}^k \exp(-v_i(f_{t-1};\boldsymbol{x}))}) \notag\\
&  \ge - \alpha \log(1+\sum_{i=1}^k \exp(-v_i(f_{t-1};\boldsymbol{x}))) - \beta \notag\\
&  = - \alpha \ell(\boldsymbol{v}(f_{t-1};\boldsymbol{x})) - \beta.
\end{align*}
Using the results above, we have 
\begin{align*}
L_\text{dis}(f_t;f_{t-1},\mathcal{D}) 
&=L_\text{con}(f_t;\mathcal{D})+\E_{\substack{c^+ \sim \mu\\ c_i^- \sim \mu } }
\E_{\substack{x,x^+ \sim \cD_{c^+} \\
    x_i^- \sim \cD_{{c_i^-}}}} \sum_{i=1}^k q_i(f_{t-1};\boldsymbol{x}) v_i(f_t;\boldsymbol{x})\notag\\
& \ge L_\text{con}(f_t;\mathcal{D})+\E_{\substack{c^+ \sim \mu\\ c_i^- \sim \mu } } \E_{\substack{x,x^+ \sim \cD_{c^+} \\     x_i^- \sim \cD_{{c_i^-}}}} [ -  \alpha \ell(\boldsymbol{v}(f_{t-1};\boldsymbol{x})) - \beta ]\notag\\
& = L_\text{con}(f_t;\mathcal{D}) - \alpha L_\text{con}(f_{t-1};\mathcal{D}) - \beta.
\end{align*}
which means
\begin{align*}
L_\text{con}(f_t;\mathcal{D}) \le \alpha L_\text{con}(f_{t-1};\mathcal{D}) + L_\text{dis}(f_t;f_{t-1},\mathcal{D}) + \beta.
\end{align*}
This finishes the upper bound part. Then let us consider the lower bound.  We use the inequality
\begin{align*}
-\alpha \log h - \beta'  \ge 2(1-\frac{1}{h}),
\end{align*}
where $\alpha =  \frac{2e^2}{k+e^2}$, $\beta' = -\alpha \log(1+ke^2) -\frac{2ke^2}{1+ke^2}$, $1+ke^{-2} \le h \le 1+ke^{2}$, and using $1+\sum_{i=1}^k \exp(-v_i(f_{t-1};\boldsymbol{x}))$ to replace $h$, then we have  
\begin{align*}
\sum_{i=1}^k q_i(f_{t-1};\boldsymbol{x}) v_i(f_t;\boldsymbol{x})
& \le 2 (1-q(f_{t-1};\boldsymbol{x})) \notag\\
& =  2 (1-\frac{1}{1+\sum_{i=1}^k \exp(-v_i(f_{t-1};\boldsymbol{x}))}) \notag\\
& \le - \alpha \log(1+\sum_{i=1}^k\exp(-v_{i}(f_{t-1};\boldsymbol{x})) - \beta' \notag\\
& = - \alpha \ell(\boldsymbol{v}(f_{t-1};\boldsymbol{x})) - \beta'.
\end{align*}
Using the results above, we have
\begin{align*}
L_\text{dis}(f_t;f_{t-1},\mathcal{D}) 
&=L_\text{con}(f_t;\mathcal{D})+\E_{\substack{c^+ \sim \mu\\ c_i^- \sim \mu } }
\E_{\substack{x,x^+ \sim \cD_{c^+} \\
    x_i^- \sim \cD_{{c_i^-}}}}  \sum_{i=1}^k q_i(f_{t-1};\boldsymbol{x}) v_i(f_t;\boldsymbol{x})\notag\\
& \le L_\text{con}(f_t;\mathcal{D})+\E_{\substack{c^+ \sim \mu\\ c_i^- \sim \mu } } \E_{\substack{x,x^+ \sim \cD_{c^+} \\     x_i^- \sim \cD_{{c_i^-}}}} [ -  \alpha \ell(\boldsymbol{v}(f_{t-1};\boldsymbol{x})) - \beta '] \notag\\
& = L_\text{con}(f_t;\mathcal{D}) - \alpha L_\text{con}(f_{t-1};\mathcal{D}) - \beta'.
\end{align*} 
It can be translated into
\begin{align*}
L_\text{con}(f_t;\mathcal{D}) \ge \alpha L_\text{con}(f_{t-1};\mathcal{D}) + L_\text{dis}(f_t;f_{t-1},\mathcal{D}) + \beta'.
\end{align*}
The proof has finished.
\end{proof}

\section{The Case of Multiple Negative Examples for Theorem~\ref{thm1}}
\label{appendixE}

We provide the extended version of Theorem~\ref{thm1} and its proof.

\begin{theorem}
\label{thm3}
For the contrastive continual learning involving $T\ge 2$ tasks where each task involves $k$ negative samples, the test loss of the final model $f_T$ can be bounded via a linear combination of the training losses associated with each task. More specifically, the following two bounds are applicable.

(1) Upper bound:
\begin{align*}
&L_\text{\rm test}(f_T;\cD_1,\dots,\cD_T) \le \alpha^{T-1} L_\text{\rm train}(f_1;\cD_1) \notag \\ 
&\quad + \sum_{t=2}^{T} \frac{\alpha^{T-t}}{\gamma_{t}(\lambda)}  L_\text{\rm train}(f_t;f_{t-1},\cD_t,\cD_{1:{t-1}}) +\eta , 
\end{align*}

(2) Lower bound:
\begin{align*}
&L_\text{\rm test}(f_T;\cD_1,\dots,\cD_T) \ge \alpha^{T-1} L_\text{\rm train}(f_1;\cD_1) \notag\\
&\quad + \sum_{t=2}^{T} \frac{\alpha^{T-t}}{\gamma_{t}'(\lambda)}  L_\text{\rm train}(f_t;f_{t-1},\cD_t,\cD_{1:{t-1}}) +\eta' , 
\end{align*}
where
\begin{align*}
\begin{cases}
{\alpha = \frac{2e^2}{ k+e^2}} , \\
\gamma_{t}(\lambda) = \min\left(\{\frac{1}{t}\}\cup \{\lambda k_{tj}\}_{j=1}^{t-1}\right), \\
\gamma_{t}'(\lambda) = \max \left(\{1\}\cup \{\lambda k_{tj}\}_{j=1}^{t-1}\right), \\
\eta =(2-\alpha+\alpha \log \frac{\alpha}{2}) \frac{T - 1 - T \alpha + (\alpha)^T}{(1-\alpha)^2}  \\
\qquad + \sum_{t=2}^{T} \alpha^{T-t} (1-\frac{1}{\gamma_{t}(\lambda)}) {\min}_f L_\text{\rm con}(f;\cD_t), \\
{\eta' = -{(\alpha \log(1+ke^2) +\frac{2ke^2}{1+ke^2})} \frac{T - 1 - T \alpha + (\alpha)^T}{(1-\alpha)^2}}.
\end{cases}
\end{align*}
\end{theorem}

\begin{proof}
We first proof the upper bound. For models $f_t$ and $f_{t-1}$, we have
\begin{align*}
{ L_\text{dis}(f_t;f_{t-1},\mathcal{D}_{1:t-1})}
& = \E_{\substack{c^+ \sim \mu_{1:{t-1}}\\ c^- \sim \mu_{1:{t-1}} } } \E_{\substack{x,x^+ \sim \cD_{c^+} \\ 
x^- \sim \cD_{{c^-}}}}  
[{-}{{\boldsymbol{p}}(f_{t-1};\boldsymbol{x}})\cdot{\log{\boldsymbol{p}}(f_t;\boldsymbol{x})}] \notag\\
& = \sum_{j=1}^{t-1} k_{tj} 
\E_{\substack{c^+ \sim \mu_j\\ c^- \sim \mu_j } } \E_{\substack{x,x^+ \sim \cD_{c^+} \\ 
x^- \sim \cD_{{c^-}}}}
[{-}{{\boldsymbol{p}}(f_{t-1};\boldsymbol{x}})\cdot{\log{\boldsymbol{p}}(f_t;\boldsymbol{x})}] \notag\\
& = \sum_{j=1}^{t-1} k_{tj} { L_\text{dis}(f_t;f_{t-1},\mathcal{D}_j)}.
\end{align*}

Then we can write the training loss as
\begin{align*}
{ L_\text{train}(f_t;f_{t-1},\mathcal{D}_t, \mathcal{D}_{1:t-1})} 
&=L_\text{con}(f_t;\mathcal{D}_t)+ \lambda { L_\text{dis}(f_t;f_{t-1},\mathcal{D}_{1:t-1})} \notag\\
&= L_\text{con}(f_t;\mathcal{D}_t) + \lambda \sum_{j=1}^{t-1} k_{tj} { L_\text{dis}(f_t;f_{t-1},\mathcal{D}_j)}.
\end{align*}
for task $t \ge 2$, and $L_\text{train}(f_1;\mathcal{D}_1)=L_\text{con}(f_1;\mathcal{D}_1)$.
According to the proof of Lemma~\ref{lemma1}, for models $f_t$ and $f_{t-1}$, data distribution $\mathcal{D}_j$ $ (j \le t)$, we have
\begin{align*}
L_\text{con}(f_t,\mathcal{D}_j) \le \alpha L_\text{con}(f_{t-1},\mathcal{D}_j) + { L_\text{dis}(f_t;f_{t-1},\mathcal{D}_j)} + \beta.
\end{align*}
Denote  $\gamma_{t}(\lambda) = \min\left(\{\frac{1}{t}\}\cup \{\lambda k_{tj}\}_{j=1}^{t-1}\right)$ for task $t \ge 2$. { According to the equations above,} we have
\begin{align*}
& L_\text{test}(f_T;\mathcal{D}_{1:T}) \notag\\
& = L_\text{con}(f_T;\mathcal{D}_T) + \sum_{t=1}^{T-1}  L_\text{con}(f_T,\mathcal{D}_t) \notag\\
& \le L_\text{con}(f_T;\mathcal{D}_T) + \sum_{t=1}^{T-1} [{ L_\text{dis}(f_T;f_{T-1},\mathcal{D}_t)} + \alpha L_\text{con}(f_{T-1};\mathcal{D}_t) + \beta]\notag \\
& \le (\frac{1}{\gamma_T(\lambda)} - \frac{1}{\gamma_T(\lambda)} +1)L_\text{con}(f_T;\mathcal{D}_T) + \frac{\lambda}{\gamma_T(\lambda)} \sum_{t=1}^{T-1} k_{Tt} { L_\text{dis}(f_T;f_{T-1},\mathcal{D}_t)} + \sum_{t=1}^{T-1} [\alpha L_\text{con}(f_{T-1};\mathcal{D}_t) + \beta]\notag \\
& \le  \frac{1}{\gamma_T(\lambda)} { L_\text{train}(f_T;f_{T-1},\mathcal{D}_T,\mathcal{D}_{1:T-1})}+(1- \frac{1}{\gamma_T(\lambda)}) {\min}_f L_\text{con}(f;\mathcal{D}_T) \notag\\
& \qquad + (T-1)\beta + \alpha L_\text{test}(f_{T-1};\mathcal{D}_{1:{T-1}})\notag\\
&\qquad \vdots \notag\\
& \le  \alpha^{T-1} L_\text{train}(f_1;\mathcal{D}_1) +\sum_{t=2}^{T} \frac{\alpha^{T-t}}{\gamma_{t}(\lambda)} { L_\text{train}(f_t;f_{t-1},\mathcal{D}_t,\mathcal{D}_{1:t-1})}+\eta. 
\end{align*}
where
${ \alpha = \frac{2e^2}{k+e^2}}$, $\eta =(2-\alpha+\alpha \log \frac{\alpha}{2}) \frac{T - 1 - T \alpha + (\alpha)^T}{(1-\alpha)^2}+\sum_{t=2}^{T} \alpha^{T-t} (1-\frac{1}{\gamma_{t}(\lambda)}) {\min}_f L_\text{con}(f;\mathcal{D}_t)$.

Let us prove the lower bound. According to the proof of Lemma~\ref{lemma1}, for models $f_t$ and $f_{t-1}$, and data distribution $\mathcal{D}_j$ $ (j \le t)$, we have
\begin{align*}
L_\text{con}(f_t,\mathcal{D}_j) \ge \alpha L_\text{con}(f_{t-1},\mathcal{D}_j) + { L_\text{dis}(f_t;f_{t-1},\mathcal{D}_j)} + \beta' .
\end{align*}
Denote $\gamma_{t}'(\lambda) = \max\left(\{1\}\cup \{\lambda k_{tj}\}_{j=1}^{t-1}\right)$ for task $t \ge 2$. The proof is similar to that of the upper bound. { Similarly, we have}
\begin{align*}
& L_\text{test}(f_T;\mathcal{D}_{1:T})\notag \\
& = L_\text{con}(f_T;\mathcal{D}_T) + \sum_{t=1}^{T-1}  L_\text{con}(f_T,\mathcal{D}_t) \notag\\
& \ge L_\text{con}(f_T;\mathcal{D}_T) + \sum_{t=1}^{T-1} [{ L_\text{dis}(f_T;f_{T-1},\mathcal{D}_t)} + \alpha L_\text{con}(f_{T-1};\mathcal{D}_t) + \beta'] \notag\\
& \ge \frac{1}{\gamma_t'(\lambda)}[ L_\text{con}(f_T;\mathcal{D}_T) + \lambda \sum_{t=1}^{T-1} k_{Tt} { L_\text{dis}(f_T;f_{T-1},\mathcal{D}_t)}] + \sum_{t=1}^{T-1} [\alpha L_\text{con}(f_{T-1};\mathcal{D}_t) + \beta'] \notag\\
& = \frac{1}{\gamma_t'(\lambda)} { L_\text{train}(f_T;f_{T-1},\mathcal{D}_T,\mathcal{D}_{1:T-1})} + \alpha \sum_{t=1}^{T-1} L_\text{con}(f_{T-1};\mathcal{D}_t) + (T-1)\beta'\notag\\
& =\frac{1}{\gamma_t'(\lambda)} { L_\text{train}(f_T;f_{T-1},\mathcal{D}_T,\mathcal{D}_{1:T-1})} + \alpha L_\text{test}(f_{T-1};\mathcal{D}_{1:{T-1}}) + (T-1)\beta' \notag\\
& \ge  \alpha^{T-1} L_\text{train}(f_1;\mathcal{D}_1)+\sum_{t=2}^{T} \frac{\alpha^{T-t}}{\gamma_t'(\lambda)} { L_\text{train}(f_t;f_{t-1},\mathcal{D}_t,\mathcal{D}_{1:t-1})} +\eta'.
\end{align*}
where
${ \alpha = \frac{2e^2}{k+e^2}}$, ${ \eta' = -{(\alpha \log(1+ke^2) +\frac{2ke^2}{1+ke^2})} \frac{T - 1 - T \alpha + (\alpha)^T}{(1-\alpha)^2}}$.
\end{proof}

\end{document}